\documentclass[letterpaper]{article} % DO NOT CHANGE THIS
\usepackage{aaai2026}  % DO NOT CHANGE THIS
\usepackage{times}  % DO NOT CHANGE THIS
\usepackage{helvet}  % DO NOT CHANGE THIS
\usepackage{courier}  % DO NOT CHANGE THIS
\usepackage[hyphens]{url}  % DO NOT CHANGE THIS
\usepackage{graphicx} % DO NOT CHANGE THIS
\usepackage{natbib}  % DO NOT CHANGE THIS AND DO NOT ADD ANY OPTIONS TO IT
\usepackage{caption} % DO NOT CHANGE THIS AND DO NOT ADD ANY OPTIONS TO IT
\usepackage{algorithm}
\usepackage{algorithmic}

\usepackage{latexsym}
\usepackage{amssymb}
\usepackage{amsmath}
\usepackage{amsthm}
\usepackage{booktabs}
\usepackage{enumitem}
\usepackage{soul}
\usepackage{color}
\usepackage{comment}

\usepackage[acronym]{glossaries}

\newacronym{aip}{AIP}{Analogical Inference Principle}

\newtheorem{theorem}{Theorem}
\newtheorem{lemma}[theorem]{Lemma}
\newtheorem{corollary}[theorem]{Corollary}
\newtheorem{proposition}[theorem]{Proposition}
\newtheorem{definition}[theorem]{Definition}
\newtheorem{remark}[theorem]{Remark}

\newcommand*\err[0]{\mbox{err}} %pour err
\newcommand*\albl[1]{\overline{#1}}

\usepackage{todonotes}

\usepackage{newfloat}
\usepackage{listings}
\DeclareCaptionStyle{ruled}{labelfont=normalfont,labelsep=colon,strut=off} % DO NOT CHANGE THIS
\floatstyle{ruled}
\newfloat{listing}{tb}{lst}{}
\floatname{listing}{Listing}
\title{Generalizing Analogical Inference from Boolean to Continuous Domains}
\author {
   Francisco Cunha\textsuperscript{\rm 1},
   Yves Lepage\textsuperscript{\rm 2},
   Miguel Couceiro\textsuperscript{\rm 1,\rm 3} 
   Zied Bouraoui\textsuperscript{\rm 4},
}
\affiliations {
    \textsuperscript{\rm 1}Dep. Mathematics, IST, U. Lisboa, Portugal\\
    \textsuperscript{\rm 2}Waseda University, IPS, Japan\\
    \textsuperscript{\rm 3}INESC-ID, DEI, IST, U.Lisboa, Portugal\\
    \textsuperscript{\rm 4}CRIL CNRS, U.Artois, France\\
    francisco.vicente.e.cunha@tecnico.ulisboa.pt, yves.lepage@waseda.jp, miguel.couceiro@inesc-id.pt, bouraoui@cril.fr
}

\usepackage{bibentry}
\begin{document}

\maketitle

\begin{abstract}
Analogical reasoning is a powerful inductive mechanism, widely used in human cognition and increasingly applied in artificial intelligence. Formal frameworks for analogical inference have been developed for Boolean domains, where inference is provably sound for affine functions and approximately correct for functions close to affine. These results have informed the design of analogy-based classifiers. However, they do not extend to regression tasks or continuous domains. In this paper, we revisit analogical inference from a foundational perspective. We first present a counterexample showing that existing generalization bounds fail even in the Boolean setting. We then introduce a unified framework for analogical reasoning in real-valued domains based on parameterized analogies defined via generalized means. This model subsumes both Boolean classification and regression, and supports analogical inference over continuous functions. We characterize the class of analogy-preserving functions in this setting and derive both worst-case and average-case error bounds under smoothness assumptions. Our results offer a general theory of analogical inference across discrete and continuous domains.

\end{abstract}

% Uncomment the following to link to your code, datasets, an extended version or similar.
% You must keep this block between (not within) the abstract and the main body of the paper.
%  \begin{links}
% % %    \link{Code}{https://aaai.org/example/code}
% %  %   \link{Datasets}{https://aaai.org/example/datasets}
%      \link{Extended version}{https://arxiv.org/pdf/2511.10416}
%  \end{links}

\section{Introduction}
Analogical reasoning seeks to identify structural similarities between different situations or objects, often formulated through analogical proportions of the form $a : b :: c : d$. Such reasoning has proven effective across domains ranging from case-based reasoning and classification to transfer learning and knowledge graph construction. In particular, it offers an appealing inductive principle in settings where direct supervision is limited, as it allows the inference of labels or attributes from analogically related examples.

The items $a, b, c, d$ are supposed to be described by a set of related attributes and represented by tuples of attribute values.
Attributes may be Boolean, nominal (i.e., finite attribute domains), or real-valued (e.g., embeddings). Modeling analogical reasoning has a long tradition in artificial intelligence, particularly in logic \cite{aristotle2011nicomachean,prade2013analogical}, cognitive modeling \cite{gentner1983structure,hofstadter2013surfaces}, and machine learning \cite{hu2023incontext}. In recent years, interest in formal analogical inference has grown, driven by its applications in few-shot learning \cite{DBLP:conf/icml/HwangGS13}, transfer learning \cite{Badra2020,CornuejolsMO20}, and interpretable AI \cite{DBLP:conf/mdai/Hullermeier20}. Analogical reasoning has also played a central role in natural language processing, particularly in morphological analysis and word embedding evaluation, where analogies such as “king is to queen as man is to woman” have served as both tasks and benchmarks \cite{mikolov_naacl_13,DBLP:conf/coling/BouraouiJS18,gladkova_naacl_2016}. More recently, neural models have been designed to detect and explain analogical relations in textual data, including analogies between sentences, concepts, and procedures \cite{DBLP:conf/acl/UshioASC20,DBLP:conf/emnlp/JacobSS23,DBLP:conf/emnlp/KumarS23}.

Foundational results characterizing the class of functions compatible with Boolean analogical inference were established in \cite{CouceiroHPR17}, and later works extended analogical classification techniques to nominal and numerical domains \cite{BounhasP23,CouceiroLMPR20}. Theoretical results, notably those of \cite{CouceiroHPR18}, have established that analogical inference is exact ({\it i.e.,} error-free) if and only if the labeling function is affine. More generally, they showed that if a Boolean function is $\varepsilon$-close to an affine function, then the probability of making an incorrect analogical prediction is bounded above by $4\varepsilon$, where the probability is taken over the random choice of training sets. This result has been influential in supporting the soundness of analogy-based classifiers in Boolean settings and later extended to a Galois theory of analogical classifiers \cite{Galois2024}.

Despite this progress, two important limitations remain.  First, these results are limited to discrete attribute spaces and binary classification tasks. In practice, many real-world applications involve real-valued features and regression tasks, settings that are not covered by current theory. To the best of our knowledge, no prior work has addressed analogical reasoning in the context of regression. Secondly, the theoretical guarantees established in the Boolean setting such as the generalization bound proposed in \cite{CouceiroHPR18,Galois2024} fail to extend to real-valued functions. This limits the application of these results to handle regression. 

In this work, we address these shortcomings by revisiting analogical inference from both theoretical and practical perspectives.  Firstly, we construct an explicit counterexample demonstrating that the main generalization bound proposed in \cite{CouceiroHPR18} fails to hold, even for Boolean functions that are arbitrarily close to affine. This challenges a foundational assumption in the existing theoretical framework. Secondly, we extend analogical inference to real-valued domains by introducing a unified model based on parameterized analogies defined via generalized means. This formulation naturally subsumes classical notions such as additive and geometric analogies and supports analogical reasoning in regression tasks. Thirdly, we introduce functional distances tailored to this generalized setting and derive both uniform and probabilistic bounds on inference error, yielding worst-case and average-case guarantees under smoothness assumptions. Finally, we characterize the class of continuous functions that preserve analogical structures, showing that they correspond to a family of generalized-power functions with well-defined structural properties.

%\todo{Put explicitely that this is theoretical work, Put contributions in bullet points and improve roadmap!!}

%\begin{comment}
The remainder of this paper is organized as follows. Section \ref{sec:background} reviews foundational work on analogical inference, including Boolean and nominal settings, and discusses the concept of analogy-preserving functions. Section \ref{sec:counterexp} presents a counterexample that falsifies a widely cited generalization bound, even in the Boolean case. Section \ref{sec:apf} introduces a generalized framework for analogical proportions based on parameterized means and redefines analogy-preserving functions in continuous domains. In Section \ref{sec:reg}, we formalize analogical inference for regression, characterize the class of compatible functions, and establish performance guarantees. Section \ref{sec:classification} shows how the Boolean case is recovered. We conclude with a summary and outline directions for future research. 
%For clarity and continuity, proofs are provided in the extended version\footnote{\url{https://hal.science/hal-05359290}}. 
%\end{comment}

%{\bf Recheck:} and maybe add a few lines...

\section{Background and Related Work}
\label{sec:background}
Analogical inference builds on the foundational idea of analogical proportions, quaternary relations of the form $a : b :: c : d$, which express that the transformation from $a$ to $b$ is analogous to that from $c$ to $d$. This concept has been studied both in logic and learning, where it forms the basis for inference schemes applicable to classification, regression, and relational reasoning. In formal settings, analogical proportions are evaluated componentwise, and their algebraic properties have been studied over Boolean and nominal domains.

\subsection{Analogical Inference from Boolean to General Cases}
\label{subsection:to:general:cases}

In the Boolean case, the most well-known models of analogy  are {\it Klein's model} \cite{klein_anthrolopology_1983}  of Boolean analogy\footnote{Here the columns are precisely the tuples $(a,b,c,d)$ for which the analogy holds.}
and the so-called {\it
minimal model} \cite{miclet_sqaru_2009}\footnote{Note that $M$ contains only patterns of the form $x:x::y:y$ and $x:y::x:y$, for $x,y\in \{0,1\}$.
%\YLN{Il me semble que $R_2$ n'a pas été introduit avant.}%
}.
\[
K :=
\begingroup
\setlength{\arraycolsep}{1.5pt} % Adjust this value as needed
\begin{pmatrix}
0 & 1 & 1 & 0 & 1 & 0 & 0 & 1 \\
0 & 1 & 0 & 1 & 0 & 1 & 0 & 1 \\
0 & 0 & 1 & 1 & 0 & 0 & 1 & 1 \\
0 & 0 & 0 & 0 & 1 & 1 & 1 & 1
\end{pmatrix}
\quad
M :=
\begin{pmatrix}
0 & 1 & 1 & 0 & 0 & 1 \\
0 & 1 & 0 & 1 & 0 & 1 \\
0 & 0 & 1 & 0 & 1 & 1 \\
0 & 0 & 0 & 1 & 1 & 1
\end{pmatrix}
\endgroup
\]
%\YLN{Après, dans la section~\ref{sec:classification}, il est fait mention de $K$ et $M$ pour noter ces modèles, mais cela n'est pas introduit ici.}

%\YL{J'ai supprimé un paragraphe qui était un doublon.}
%\begin{comment}
%The problem of finding an $x\in \{0,1\}$ such that $a:b::c:x$ holds, does not always have a solution.
%Indeed, neither $0:1::1:x$ nor $1:0::0:x$ has a solution (since $0111$, $0110$, $1000$, $1001$ are not valid patterns for an analogical proportion). In fact, a solution exists if and only if $(a \equiv b) \vee (a \equiv c)$ holds.
%When a solution exists, it is unique and is given by $x = c \equiv (a \equiv b)$.
%This corresponds to the original view advocated by S. Klein \cite{Klein1983}, who however applied the latter formula even to the cases $0:1::1:x$ and $1:0::0:x$, where it yields $x=0$ and $x=1$ respectively.
%\end{comment}

In the Boolean case, the problem of finding an $x\in \{0,1\}$ such that $a:b::c:x$ holds, does not always have a solution.
For instance, neither $0:1::1:x$ nor $1:0::0:x$ has a solution in the minimal model $M$ (since $0111$, $0110$, $1000$, $1001$ are not valid patterns for an analogical proportion). In fact,
a solution exists if and only if $(a \equiv b) \vee (a \equiv c)$ holds.
When a solution exists, it is unique and is given by $x = c \equiv (a \equiv b)$
%\YL{Je suggère en toute humilité : }
(see also \cite{lepage_iarml_2023}).
This corresponds to the original view advocated by S. Klein \cite{Klein1983}, who however applied the latter formula even to the cases $0:1::1:x$ and $1:0::0:x$, where it yields $x=0$ and $x=1$ respectively.

In the nominal case, the situation is similar. The analogical proportion $a:b::c:x$ may have no solution ({\it e.g.,}  in the minimal model, $s:t::t:x$ has no solution as soon as $s \neq t$), and otherwise (if $a=b$ or $a=c$) the solution is unique, and is given by $x = b$ if $a = c$ and  $x = c$ if $a = b$. Namely, the solutions of $s:t::s:x$, $s:s::t:x$, and $s:s::s:x$ are $x=t$, $x=t$, and $x=s$, respectively.
This motivates the following inference pattern, first  formalized by \cite{SY06}; see also \cite{BouPraRicECAI2014}:
\begin{align}
\label{eq:inf-pattern}
\frac{\forall i \in \{1, \dots, m\},\,\,\, a_i : b_i :: c_i : d_i \text{~holds}}{a_{m+1} : b_{m+1} :: c_{m+1} : d_{m+1} \text{~holds}}
\end{align}
 which generalizes analogical inference over attribute vectors enabling to compute $d_{m+1}$, provided that 
$a_{m+1} : b_{m+1} :: c_{m+1} :x$ has a solution. 
This pattern expresses a rather bold inference which amounts to saying that if the representations of four items are in analogical proportion on $m$ attributes, they should remain in analogical proportion with respect to their labels. We can restrict ourselves to binary labels, since a multiple class prediction can be obtained by solving a series of binary class problems.  A key question, addressed in this paper, is to characterize the settings under which this inference is valid.

\subsection{Analogy-preserving functions and analogy based classification.} The analogical inference pattern implicitly relies on the assumption that labels are functionally determined by the attribute values.  More precisely, there exists some unknown function $f$ such that for any item $\mathbf{e} = (e_1, \dots, e_n)$, the label is given by $e_{n+1} = f(e_1, \dots, e_n)$. The function $f$ can be viewed of as a classifier that assigns a (unique) class to each item based on its $n$ attributes. 
Since the solutions of analogical equations (when they exist) are unique, the inference pattern \eqref{eq:inf-pattern} can equivalently be formulated as follows:
\begin{align}
\label{eq:inf-pattern2}
\begin{array}{ccccccc}
a_1 & \cdots & a_i & \cdots & a_n &f(\mathbf{a})\\
b_1 & \cdots &b_i &\cdots & b_n &f(\mathbf{b})\\
c_1 & \cdots & c_i &\cdots &c_n &f(\mathbf{c})\\
\hline
d_1 & \cdots &d_i &\cdots &d_n &f(\mathbf{d})\\
\end{array}
\end{align}
where $\mathbf{a} = (a_1, \dots, a_n)$, $\mathbf{b} = (b_1, \dots, b_n)$, $\mathbf{c} = (c_1, \dots, c_n)$ and $\mathbf{d} = (d_1, \dots, d_n)$ are instances (objects, items, etc.).

Assuming that the latter four instances are in
analogical proportion for each of the $n$ attributes describing them, and that the class labels are known for $\mathbf{a}, \mathbf{b}, \mathbf{c}$ but unknown for $\mathbf{d}$, then one may infer that the label for $\mathbf{d}$ as a solution
of an analogical proportion equation \cite{BounhasPR17a,CouceiroHPR17}. 
The effectiveness of this analogical inference rule led to several studies that aimed to determine which classifiers were compatible with the \acrfull{aip} \cite{miclet_jair_08,BounhasPR17a,CouceiroHPR17,CouceiroHPR18,CouceiroLMPR20,DBLP:journals/ijar/BounhasP24,Galois2024}. 

In the case of Boolean attributes, a key result has been established in \cite{CouceiroHPR17}, where it was shown that
the set of functions for which analogical inference is sound,
i.e., no error occurs, are the \emph{analogy-preserving} (AP) functions, which coincide exactly with the set of affine Boolean functions. More precisely, they showed the following result.

\begin{theorem}\label{AP_is_L}
The class of AP functions is exactly the class $\mathcal{L}$ of affine functions, {\it i.e.}, functions $f:\mathbb{B}^n\rightarrow \mathbb{B}$ of the form $f=a_1x_1+\ldots+a_nx_n+b$, for $a_i\in \mathbb{B}$ and where $+$ is the addition over the 2-element field $\mathbb{B}$.
\end{theorem}

Moreover, when the function is close to being affine, it was also shown that the prediction accuracy remains high \cite{CouceiroHPR18}. Assuming a uniform distribution and the Hamming distance $d$ on the Boolean function space $\bigcup_{n>0}\mathbb{B}^{\mathbb{B}^n}$, the classification performance remained high for classifiers close to the AP functions.   

\begin{definition}\label{def:root}Given a sample set $S \subseteq \mathbb{B}^n$ and a function $f\in \mathbb{B}^{\mathbb{B}^n}$,  
the  {\it analogical
root} \cite{HugPraRicSerECAI2016} of a given element $\mathbf{x} \in \mathbb{B}^n$, denoted by $\mathbf{R}_S(\mathbf{x},f)$, is the set 
{\small
$$
\{(\mathbf{a}, \mathbf{b}, \mathbf{c}) \in S^3  \, :\,
\mathbf{a} : \mathbf{b} :: \mathbf{c} : \mathbf{x} \mbox{ and } solvable(f(\mathbf{a}), f(\mathbf{b}), f(\mathbf{c}))\}.$$}
The {\it analogical extension} $\mathbf{E}_S(f)$ of $S$ w.r.t. $f$ is thus defined as the set of all those $\mathbf{x} \in \mathbb{B}^m$ such that $\mathbf{R}_S(\mathbf{x},f) \neq \emptyset$. 
\end{definition}
 Clearly, $S \subseteq \mathbf{E}_S(f)$ since $\mathbf{a}:\mathbf{a}::\mathbf{a}:\mathbf{a}$ always holds.
Define 
$\err_{S,f}= P(\{\mathbf{x} \in \mathbf{E}_S(f) \setminus S \,|\, \albl{\mathbf{x}}_{S,f}\neq f(\mathbf{x})\} ),$ where $\albl{\mathbf{x}}_{S,f}$ is the predicted label by AIP of $\mathbf{x}$. 
For classification purposes, we are thus interested in  $\mathbf{x}\in\mathbf{E}_S(f) \setminus S$ whose 
predicted label $\albl{\mathbf{x}}_{S,f}$ is $f(\mathbf{x})$, and
 $\mathbf{E}_S(f)$ is is said to be {\it sound} if $\err_{S,f} = 0$.

\begin{theorem}
\label{bound2}
Let $\varepsilon \in [0,\frac{1}{2}]$, and let $\delta \in [0,1]$. Consider the uniform distance $\rm{d}'$  on $\mathbb{B}^{\mathbb{B}^m}$ given by $${\rm d}(f,g)=\frac{\vert\{\mathbf{x}\in \mathbb{B}^m~:~f(\mathbf{x})\neq g(\mathbf{x})\}\vert}{2^m}.$$ If ${\rm d}(f,\mathcal{L})=\min_{g\in \mathcal{L}}{\rm d}(f,g)<\varepsilon$, then $P(\err_{S,f} > \delta) \leq 4 \varepsilon \cdot (1-\delta)$.
\end{theorem}

Note that this result establishes a tight connection between inference errors and functional distances.

\subsection{Limitations and Drawbacks}
When attributes are valued on finite domains  $X_i$, i.e. nominal case (which includes the Boolean case), the problem of characterizing analogy-preserving functions $f\colon {\bf X}\to X$, for ${\bf X}=X_1\times \cdots \times X_n$,  has been partially addressed for binary classification ($|X|\leq 2$). This has led to the definition of {\it hard AP functions} (HAP) $f \colon \mathbf{X} \to X$, which are either ``essentially unary'' or ``quasi-linear'', {\it i.e.,} for which there exist $\varphi \colon \{0,1\} \to X$ and $\varphi_i \colon X_i \to \{0,1\}$ ($1 \leq i \leq m$) such that 
$f = \varphi(\varphi_1(x_1) \oplus \dots \oplus \varphi_n(x_n))$. 
However,  this characterization is limited to nominal attributes, binary labels and constrained models of analogy, typically those satisfying the patterns $(x,x,y,y)$ and $(x,y,x,y)$. In addition, prior work often assumes a decomposable model of analogy, which may not hold in real-world data. These limitations have been critically examined in recent work, such as \cite{DBLP:journals/ijar/BounhasP24}, which revisits analogical proportions and inference under relaxed assumptions.

These limitations motivate the reformulation of classification tasks using analogy-based regression. In the next section, we construct a counterexample that demonstrates the failure of the error bound from \cite{CouceiroHPR18}, even in the Boolean setting. We will then introduce an analogical inference principle tailored to regression and continuous domains, enabling extensions to multi-class settings. This framework yields new theoretical guarantees under reasonable smoothness and distributional assumptions.

\section{An Illustrative Counter-example}
\label{sec:counterexp}
In this section, we provide a counterexample to the generalization bound, which asserts that analogical inference remains accurate for Boolean functions close to the affine class $\mathcal{L}$. Let $\mathbb{B}=\{0,1\}$, we exhibit a concrete counterexample to Theorem~\ref{bound2} for $\delta=0$, by providing a Boolean function $f:\mathbb{B}^m\to\mathbb{B}$ such that $d(f,\mathcal{L})\leq \varepsilon$ but $P(\err_{S,f}>0)> 4\varepsilon$. %where  $err_{S,f}>0$ is short for $\{S\in\mathcal{P}({\mathbb{B}}^n)\,|\,\{x\in\mathbb{B}^n\,|\,\overline{x}_{S,f}\not=f(x)\}\not=\emptyset\}\,$ and $P'$ is the uniform probability on $\mathcal{P}(\mathbb{B}^n)$ (i.e. the normalized counting measure on this finite set).

\begin{definition}  Let $f\colon \mathbb{B}^4\to\mathbb{B}$ be the function given by: 
\begin{align}
f(\mathbf{x}) =
\begin{cases}
1 & \text{if } \mathbf{x} =\mathbf{1} \\
0 & \text{otherwise},
\end{cases}
\end{align}
%whose value is $0$ on every $\mathbf{x}\in \mathbb{B}^4\setminus \{(1,1,1,1)\}$, and $f(1,1,1,1)=1$.
where $\mathbf{1}$ denotes the constant-1 tuple $(1,1,1,1)$. Since the constant-0 Boolean function $h\colon \mathbb{B}^4\to\mathbb{B}$ ({\it i.e.,} $h(\mathbf{x})=0$, for every $\mathbf{x}\in \mathbb{B}^4$) is affine, it is not difficult to see that ${\rm d}(f,\mathcal{L})=\min_{g\in \mathcal{L}}{\rm d}(f,g)={\rm d}(f,h)=\frac{1}{16}$.
\end{definition}

\paragraph{Failure of analogical inference.}
Let $S\subseteq \mathbb{B}^4\setminus\{\mathbf{1}\}$, and suppose that $\mathbf{1}\in E_S(f)$. Then $\overline{\mathbf{1}}_{S,f}=0$, since $f|_S\equiv0$ and the only solution of $0:0::0:x$ for both Klein's and the minimal  models is $x=0$.
This means that if for a given $S\subseteq \mathbb{B}^4$ we have $\mathbf{1}\in \mathbf{E}_S(f)\setminus S$, then $\mathbf{1}\in\{\mathbf{x}\in\mathbb{B}^4\,|\,\overline{\mathbf{x}}_{S,f}\neq f(\mathbf{x})\}$ (in particular this set is nonempty), which means that $S\in\{T\in \mathcal{P}(\mathbb{B}^4)
\,|\,\err_{T,f}>0\}$. So we get 
% $$\{S\in \mathcal{P}(\mathbb{B}^4)
% \,|\,err_{S,f}>0\}\supseteq\{S\in \mathcal{P}(\mathbb{B}^4)\,|\,(1,1,1,1)\in E_{S,f}\setminus S\}$$
% $$\Rightarrow P'(\{S\in \mathcal{P}(\mathbb{B}^4)
% \,|\,err_{S,f}>0\})\geq P'(\{S\in \mathcal{P}(\mathbb{B}^4)\,|\,(1,1,1,1)\in E_{S,f}\setminus S\})$$
% $$\Rightarrow P'(err_{S,f}>0)\geq2^{-16}|\{S\in \mathcal{P}(\mathbb{B}^4)\,|\,(1,1,1,1)\in E_{S,f}\setminus S\}|$$
$$\{S \in \mathcal{P}(\mathbb{B}^4) \,|\, \err_{S,f} > 0\} \supseteq \\
\{S \in \mathcal{P}(\mathbb{B}^4) \,|\, \mathbf{1} \in \mathbf{E}_S(f) \setminus S\}$$
which implies that 
\begin{equation*}
\{S \in \mathcal{P}(\mathbb{B}^4) \,|\, \err_{S,f} > 0\} \geq
\{S \in \mathcal{P}(\mathbb{B}^4) \,|\, \mathbf{1} \in \mathbf{E}_S(f) \setminus S\}.
\end{equation*}
Thus
$P\big(\{S \in \mathcal{P}(\mathbb{B}^4) \,|\, \err_{S,f} > 0\}\big)$ is greater %\YLI{than} 
or equal to
$$
P\big(\{S \in \mathcal{P}(\mathbb{B}^4) \,|\, \mathbf{1} \in \mathbf{E}_S(f) \setminus S\}\big).
$$
Hence,
\begin{equation}
P(\err_{S,f} > 0) \geq \\
\frac{\left| \left\{ S \in \mathcal{P}(\mathbb{B}^4) \,\middle|\, \mathbf{1} \in \mathbf{E}_S(f) \setminus S \right\} \right|}{2^{16}}.
\label{eq:prob-lowerbound}
\end{equation}
The first inequality follows from the monotony of $P$ (which is the uniform probability measure over $\mathcal{P}(\mathbb{B}^4)$) and the second from its definition. We can compute this lower bound for $P(\err_{S,f}>0)$ using a brute force algorithm (Algorithm 1). It consists in checking, for each of the $2^{(2^{n}-1)}$ (in our case $n=4$) subsets $S_i$ of $\mathbb{B}^n\setminus\{\mathbf{1}\}$, whether  there is any triple $(\mathbf{a},\mathbf{b},\mathbf{c})\in S^3$ such that $\mathbf{a}:\mathbf{b}::\mathbf{c}:\mathbf{1}$, and, if so, increasing the running total by $1$.\footnote{Here $[T[i]|\mathbf{1}]$ is of the form $(\mathbf{a},\mathbf{b},\mathbf{c},\mathbf{1})$, for $T[i]=(\mathbf{a},\mathbf{b},\mathbf{c})$.} Fortunately, since in this example $n=4$, the algorithm is still manageable and takes about 30 seconds to run. We obtain from Algorithm~\ref{alg:CE} the lower bound  $$P(\err_{S,f}>0)\geq0.42,$$ which  contradicts the upper bound $$P(\err_{S,f}>0)\leq 4{\rm d}(f,\mathcal{L})= 4\frac{1}{16}=0.25$$ of Theorem~\ref{bound2} (using $\delta=0)$.

\begin{algorithm}[tb]
\caption{Estimate of the proportion of subsets with analogical error. Here, $\text{tupleslo}(A, n)$ outputs all $n$-tuples over $A$ in lexicographic order,  $\text{subsets}(A)$ outputs the set of subsets of $A$, whereas $\text{analogyQ}$ receives as an input a quadruple of $m$-tuples, checks whether each component of the quadruple constitutes an analogy under the minimal model $M$, and outputs True if so, and False, otherwise.}
\label{alg:CE}
\textbf{Input}: Integer $n$ (dimension of Boolean vectors) \\
\textbf{Output}: Estimated probability $P(\text{err}_{S,f} > 0)$ \\
\begin{algorithmic}[1]
\STATE Let $D \gets \text{tupleslo}([ 0, 1 ], n)$
\STATE Let $R \gets \text{subsets}(D \setminus \{\mathbf{1}\})$
\STATE $c \gets 0$
\FOR{each $S \in R$ such that $|S| \geq 3$}
\STATE Let $T \gets \text{tupleslo}(S, 3)$
\STATE $stop \gets \text{False}$, $i \gets 0$
\WHILE{not $stop$ and $i < |T|$}
    \STATE $stop \gets \text{analogyQ}([T[i]|\mathbf{1}])$
    \STATE $i \gets i + 1$
\ENDWHILE
\IF{$stop$}
    \STATE $c \gets c + 1$
\ENDIF
\ENDFOR
\STATE \textbf{return} $c / (2^{2^n})$
\end{algorithmic}

\end{algorithm}

\section{Analogy-Preserving Functions}
\label{sec:apf}
The generalization of analogical reasoning from Boolean to continuous domains requires a more flexible notion of analogy that can capture numeric structure. To address this, we adopt a parameterized framework based on generalized means, originally studied by Hölder  \cite{hoelder_mittel_1889}, and recently proposed for analogical inference in \cite{lepage_iarml_2024}. This approach defines analogical proportions over real-valued tuples and supports analogical inference in both classification and regression settings.%\todo{summaries the content of this section}

\subsection{Parameterized Analogical Proportions}
For domains where inputs are represented as vectors, matrices or higher-order tensors, the unified view of analogical proportions is based on 
the \emph{generalized mean in $p$}:% is defined as: 
\begin{align}
m_p(x_1,x_2,\ldots,x_n) = \lim_{r\rightarrow p} \left(\frac{1}{n}\sum_{i=1}^{n}x_i^r\right)^{1/r}.
\end{align}
With this, $(a,b,c,d)\in \mathbb{R}_+^4$ constitutes a valid analogy if there is a $p\in \mathbb{R}$  such that $m_p(a,d) = m_p(b,c)$, {\it i.e.}, the generalized mean in $p$ of the \emph{extremes} $a$ and $d$  equals the generalized mean in $p$ of the \emph{means} $b$ and $c$.

This is denoted as ``analogy in analogical power $p$'' by  $a : b ::^p c : d$, and
it was shown that it has several desirable properties, in particular, that $::^p$ is transitive and that it constitutes an equivalence relation for $p\in \mathbb{R}$. 
One of the advantages of relying on this parameterized notion on $p$, is that it naturally subsumes well known mean notions, such as the commonly used arithmetic (for $p=1$), geometric (for $p=0$) or harmonic means (for $p=-1$).
As consequence, for any four increasing positive real numbers $a,b,c$ and $d$ there exists a unique
analogical power $p$ such that $a : b ::^p c : d$ holds. Notice that such analogy can be reduced to an equivalent arithmetic analogy and that any analogical equation has a solution for increasing numbers. 

%%%here%%%%
\subsection{Analogical Roots and Extensions}

\begin{definition}
    Let $\mathbf{x}\in\mathbb{R}_+^n$, and $\mathbf{p}=(p_1,...,p_n)\in \mathbb{R}_+^n$ and $q\in \mathbb{R}_+$. Let $S$ be a finite subset of $\mathbb{R}_+^n$. The $(\mathbf{p};q)$-analogical root  $R_S(f,\mathbf{x})$ of $\mathbf{x}$ with respect to $f$ and $S$ is defined as follows:
\begin{eqnarray*}
R_S^{(\mathbf p,q)}(f,\mathbf{x})&:=&\{(\mathbf{a},\mathbf{b},\mathbf{c})\in S^3 \,|\,\mathbf{a}:\mathbf{b}::^{(\mathbf{p})} \mathbf{c}:\mathbf{x}\\
&&~\text{and}~ f(\mathbf{a})^q\leq f(\mathbf{b})^q+f(\mathbf{c})^q\}
\end{eqnarray*}
\end{definition}

%\YLN{Faut-il que les puissances $\mathbf{p}$ soient positives ? Où est-ce nécessaire~?}
%\YLN{À part le souci de généralité, pourquoi $S$ est-il nécessaire dans cette définition~? Pourquoi ne pas prendre simplement $S=\left(\mathbb{R}_+\right)^n$~?}
\begin{remark}
This definition seems rather different
than
that of Definition~\ref{def:root}. However, it is a clear generalization: indeed,  the analogical equation $a:b::^qc:d$ has a solution if and only if  $a^q\leq b^q+c^q$.  
\end{remark}

Hence  $R_S^{(\mathbf p,q)}(f,\mathbf{x})$ is the set of the triples $(\mathbf a,\mathbf b,\mathbf c)\in S^3$ that form an analogy with $\mathbf x$ in powers $\mathbf{p}= (p_1,...,p_n)$, and such that the analogical equation $f(\mathbf a):f(\mathbf b)::^qf(\mathbf c):y$ can be solved in $\mathbb{R}_+$. We denote the set of all such solutions $y$  by $\text{sol}_q(f(\mathbf a),f(\mathbf b),f(\mathbf c))$.

\begin{definition}
  Let $\mathbf{p}=(p_1,...,p_n)\in \mathbb{R}_+^n$, $S\subseteq
  \mathbb{R}_+^n$ and $q\in\mathbb{R}_+$. The $(\mathbf p;q)$-analogical extension of $S$ with respect to $f$, $E^{(\mathbf p;q)}_S(f)$, is defined as follows:
    \[E^{(\mathbf p;q)}_S(f):=\{\mathbf{x}\in\mathbb{R}_+^n\,|\,R_S(f,\mathbf{x})\not=\emptyset\}.\]
    When $\mathbf p$ and $q$ are clear from the context we shall simply write $E_S(f)$.
\end{definition}

The analogical extension of $S$ is the subset of $\mathbb{R}_+^{n}$ that can be valued using the \acrlong{aip}. If $\mathbf x\in E_S(f)$, we can assign an \textit{analogical value} to $\mathbf{x}$ (which will ideally coincide with $f(\mathbf x)$) using the \acrlong{aip}.

\subsection{Characterizing Analogy-Preserving Functions}
We now turn to the characterization of functions that preserve analogical proportions under the generalized setting.

\begin{definition}
    Let $\mathbf{p}$, $f$ and $S$ be as in the previous definition, and let $\mathbf{x}\in E_S(f)$. Define the {\it analogical value} $\overline{\mathbf{x}}_{S,f}^{(\mathbf{p},q)}$  of $\mathbf{x}$ w.r.t. 
    $\mathbf{p}$,q, $f$ and $S$, as 
    \[
    m_q\big(\text{sol}_q(f(\mathbf a),f(\mathbf b),f(\mathbf c))~| 
    ~(\mathbf{a},\mathbf{b},\mathbf{c})\in R^{(\mathbf{ p};q)}_S(f,\mathbf{x})\big)
\]
     where $m_q$ denotes the $q$-generalized mean. When clear from the context, we will simply write $\overline{\mathbf{x}}_{S,f}$.%, and $\text{sol}_q(f(\mathbf a),f(\mathbf b),f(\mathbf c)$ denotes the set of solutions of the analogical equation $f(\mathbf a):f(\mathbf b)::^qf(\mathbf c):x$.    
     
\end{definition}
%\YLN{Sur la ligne ``w.r.t. $\mathbf{p}$,q, $f$ and $S$'', au lieu de $\mathbf{p}$,q, pas plutôt $(\mathbf{p},q)$~?}

Following the same steps as in Section~\ref{sec:background}, we first seek to describe the set of $AP_{(\mathbf{p};q)}$ of  all $(\mathbf{p};q)$-analogy preserving functions, that is, functions $f$ such that for all $S\subseteq\mathbb{R}_+^n$, $\overline{x}_{S,f}=f(\mathbf{x})$, for all $\mathbf{x}\in E_S(f)$.

\begin{proposition}
\label{prop:AP}
If $f$ is continuous, then the following statements are equivalent.
\begin{enumerate}[leftmargin=20pt, itemsep=0.2em]
    \item $f\in AP_{(\mathbf p;q)}$
    \item $f$ maps analogies in powers $\mathbf{p}=(p_1,...,p_n)$ to analogies in power $q$.
\end{enumerate}
\end{proposition}

\begin{comment} 
\begin{proof}
Here, we give a sketched proof; the detailed proof is given in the Appendix.
To show sufficiency, we consider an arbitrary proportion $\mathbf a:\mathbf b::^{\mathbf p} \mathbf c: \mathbf d$ and use the sample $S=\{\mathbf a, \mathbf b, \mathbf c\}$ and use the assumption  $f\in AP_{\mathbf p, q}$ to conclude $$f(\mathbf d)=sol_q(f(\mathbf a),f(\mathbf b),f(\mathbf c)).$$ The proof also involves showing that $f(\mathbf a):f(\mathbf b)::^qf(\mathbf c):y$ has indeed a solution. 
The proof of necessity is a simple verification, using the definition of being $AP$. 
\end{proof}
\end{comment}
\begin{remark}
Observe that the underlying domain is the set of  nonnegative real numbers.
This might sound as a restriction but a good number of applications meet this condition.
A large field of application is image processing where the values on the gray channel (see MNIST data for instance~\cite{lecun_backpropagation_1989}) or RGB channels are non-negative real number (sometimes even natural numbers between 0 and 255).
Any processing of images involving numerical analogy is thus possible on this kind of data.
One may think about image completion, image reconstruction, etc.
Another example of field where representations are non-negative real numbers is, at least theoretically, word embedding models.
It has been shown that vectors representing words trained from some specific models are concentrated in an orthant of the space, which means that a rotation can make all components in the vectors non-negative~\cite{mimno_emnlp_2017}.
\end{remark}

%%%%%%%%%%%%%%%%%%%%%%%%%%%%%%%%%%%%%%%%%%%%%%%%%%%%%%%%%%%%%%%%%%%%%%%%%%%%%%%%%%%%%%%%%%%%%%%%%%%%%%%%%%%%%%%%%%%%%%%%%%%%%%%%%%%%%%%%%%%%%%%%%%%%%%%%%%%%%%%%%%

\section{Analogy-based Regression}
\label{sec:reg}

This section explicitly describes the class $AP_{(\mathbf{p};q)}$ of all $(\mathbf{p};q)$-analogy preserving functions under some natural assumptions (see Subsection~\ref{Subsec:AP}), and {provides} performance guarantees that establish a tight connection between regression errors and distances to the class of analogy preserving functions in various functional spaces (Subsection~\ref{Subsec:errors}). This not only corrects the estimates and performance guarantees provided in \cite{CouceiroHPR18}, but also generalizes the frameworks of, {\it e.g.}, \cite{BounhasPR17a,BounhasP23,DBLP:journals/ijar/BounhasP24,CouceiroHPR17,CouceiroHPR18,CouceiroLMPR20}, to positive reals and to both classification and regression tasks. 

\subsection{Explicit Description of AP functions}\label{Subsec:AP}
%Making use of lemma \ref{l2} and proposition \ref{p1}, we can prove the following theorem, which 

The main result of this section is the following explicit description of continuous analogy preserving functions.
We will make use of the following auxiliary result that essentially states that we can restrict our quest to the case of arithmetic analogies ({\it i.e.,} for $p=1$). 

\begin{lemma}\label{l2}
Let $p'_j,q'\in\mathbb{R}_+$. Define the mappings $r:\mathbb{R}_+\to\mathbb{R}_+$ such that $r(x)=\sqrt[q']{x}$, and $s:\mathbb{R}_+^n\to\mathbb{R}_+^n$ such that  $s(x_1,\ldots,x_n)=(x_1^{p'_1},\ldots,x_n^{p'_n})$. Then the following statements are equivalent.
\begin{enumerate}
    \item $f: \mathbb{R}_+^n\to\mathbb{R}_+$ maps analogies in powers $\mathbf{p}=(p_1,\ldots,p_n)$ to an analogy in power $q$.
    \item $g=r\circ f\circ s$ maps analogies in powers ${\bf p\odot\bf p'}=(p_1'p_1,\ldots,p'_np_n)$ to an analogy in power $q'q$.
\end{enumerate}
\end{lemma}

%\YLN{%
%Ni $\textbf{p}^\prime$ ni le domaine de $q$ ne sont définis dans le lemme~\ref{l2}. Ne pourrait-on pas réécrire ces choses avec des notations plus compactes définies à l'avance~? Par exemple, pour $\textbf{x} = (x_1,\ldots,x_n)$ et $\textbf{p} = (p_1,\ldots,p_n)$, poser $\textbf{x}^{\textbf{p}} = (x_1^{p_1},\ldots,x_n^{p_n})$. On aurait alors dans le théorème~\ref{Thm:explicit description} des trucs du style $\textbf{a}\cdot(\textbf{x}^{\textbf{p}}) + b$.}

\begin{comment}
\begin{proof} 
The proof of sufficiency, essentially follows from the fact that if we take an arbitrary proportion $\mathbf a:\mathbf b::^{\bf p\odot\bf p'}\mathbf c:\mathbf d$, it is also true that $$s(\mathbf a): s(\mathbf b)::^{\mathbf p}s(\mathbf c):s(\mathbf d)$$ Applying $f$ to this proportion, we get (with some algebraic manipulation) the desired result.   
For necessity, we can repeat the same procedure with $\frac{1}{p'_j}$ in the place of $p'_j$.  The detailed proof is give in the Appendix
\end{proof}

\end{comment}
Lemma~\ref{l2} is entails the characterization of AP functions. 

\begin{theorem}\label{Thm:explicit description}
Suppose $f:\mathbb{R}_+^n\to\mathbb{R}_+$ is continuous. Then the following statements are equivalent.
\begin{enumerate}
\item $f\in AP_{(\mathbf{p};q)}$, for $\mathbf{p}=(p_1,...,p_n)$.
\item $f(x_1,...,x_n)=\left({\sum_{j=1}^n a_{j}x_j^{p_j}+b}\right)^{1/q}$, for some matrix $[a_{j}]\in M_{1\times n}(\mathbb{R}_+)$ and some scalar $b\in\mathbb{R}_+$.
\end{enumerate}

\end{theorem}

%\YLN{% ``For some matrix''. C'est simplement un vecteur, non~?}

%\begin{comment}
\begin{proof}
To prove the theorem, we use Proposition \ref{prop:AP}, and show that $f$ maps analogies in powers $\mathbf{p}$ to analogies in power $q$
 if and only if $f$ has the form given in 2.

Applying the previous Lemma, it suffices to show the result for $\mathbf{p}=(1,...,1)$ and $q=1$, which is an application of Cauchy's Functional Equation, which states that a continuous additive function must be linear.
\end{proof}

%\end{comment}

\subsection{Performance Guarantees}\label{Subsec:errors}
Inspired by the statement of Theorem~\ref{bound2} that establishes a correspondence between the distance of a Boolean classifier to the class of AP Boolean functions and the probability of that classifier to make classification errors, we seek an analogous (and correct!) result in the setting of analogy based regression.  
We will make use of background on functional spaces and measure theory and we refer the reader to \cite{taylor2008} for further background.

Recall first the notion of ``$q$-distance'' between  $x,y\in\mathbb{R}_+$, that is defined by
    \[
    d_q(x,y):=\left(|x^q-y^q|\right)^\frac{1}{q}.
    \]
Formally speaking, this $q$-distance is not a {distance}, since it does not fulfill the triangle inequality for $q< 1$. However, it constitutes a {semidistance} and, as we will see, it is the natural candidate for ``distance'' when dealing with analogies in power $q$.  The analogous {distance} for functional spaces can be defined as follows.  

%\YLN{%Pour moi une métrique est une distance dans un espace vectoriel et cela a l'air d'être le cas en anglais, je viens de vérifier.
%Ici on a : $x,y\in\mathbb{R}_+$, pas dans $\left(\mathbb{R}_+\right)^n$.
%Ne pas avoir l'inégalité du triangle a à voir en premier lieu avec un distance, pas avec une métrique.
%Je pense qu'il faudrait soit utiliser uniformément \emph{metric}, soit \emph{distance}.
%}

%\YLN{%
%En toute rigueur, il est gênant d'appeler distance, ce qui n'en est pas une.
%Pour quoi ne pas dire \emph{Recall first the notion of ``$q$-semidistance''\ldots},
%et dire ensuite qu'elle ne vérifie pas l'inégalité du triangle.
%}

%  Suppose that we are interested in a domain $D\subseteq\mathbb{R}_+^n$. We can define another semimetric in the functional space $\mathbb{R}_+^D$ in the following way.

  \begin{definition}
      Let $D\subseteq\mathbb{R}_+^n$, and consider $f,g:D\to \mathbb{R}_+$. Their uniform $q$-distance is defined by
      \[d_{q,\infty}(f,g):=\sup_{\mathbf x\in D}(d_q(f(\mathbf x),g(\mathbf x))).\]
  \end{definition}

To propose an analogue to the probablistic approach proposed in \cite{CouceiroHPR17}, we will need to assume that $D\in\mathcal{B}$, where $\mathcal{B}$ is the Borel $\sigma$-algebra over $\mathbb{R}_+^n$, and consider a probability space 
$(D,\mathcal{B}_D,\mathbb{P})$, where $\mathcal{B}_D$ consists of the elements of $\mathcal{B}$ that contain $D$ (the ideal of $\mathcal{B}$ generated by $D$), and $\mathbb{P}:\mathcal{B}_D\to [0,1]$ is a probability measure. For instance, if $D$ is finite, $D\in\mathcal{B}$ and we can naturally choose $\mathbb{P}$ to be the normalized counting measure ({\it i.e.,} uniform distribution).\footnote{However, even in the discrete case, this is not the only setting that could be of interest (for example, we could have a domain with a Poisson distribution), so we will aim to develop a general theory that works for any probability space.}
    This additional structure enables us to define our desired probabilistic distance.

    \begin{definition}
        Let $f:D\to\mathbb{R}_+$ be a Borel-measurable function. Its $q$-expected value is defined as
        \[\mathbb{E}_q(f):=\left(\int_D f^q \,d\mathbb{P}\right)^\frac{1}{q}.\]
         If $g:D\to \mathbb{R}_+$ is also Borel-measurable, then their probabilistic $q$-distance is given by
          \[dist_{q}(f,g):=\mathbb{E}_q(x\mapsto d_q(f(x),g(x))).\]
    \end{definition}

     Much like $d_{q,\infty}$, $dist$ is not necessarily a metric on $\mathbb{R}_+^D$, but it is a semimetric, and the most natural notion of distance when working with analogies in power $q$.

    \begin{remark} It is noteworthy that 
        $d_{q,\infty}(f,g)$ simplifies to $\left(||f^q-g^q||_\infty\right)^\frac{1}{q}$, whereas $dist_{q}(f,g)$ simplifies to $\left(\mathbb{E}(|f^q-g^q|)\right)^\frac{1}{q}$.
       Moreover, if $D=\{\mathbf d_1,...,\mathbf d_N\}$ and $\mathbb{P}$ is the normalized counting measure, $dist_{q}(f,g)$ simplifies to $\left(\sum_{i=1}^{N}|f(\mathbf d_i)^q-g(\mathbf d_i)^q| \right)^\frac{1}{q}$.
    \end{remark}

    Thus we have constructed two semimetrics, $d_{p,\infty}$ and $dist_p$, %\YLN{rather $dist_q$?} 
    that will act as our distances in the functional space $(\mathbb{R}_+)^D$. On the one hand, the semimetric $d_{p,\infty}$ is a uniform distance and will enable us to obtain worst-case results, that is,  upper bounds for the largest possible errors when using the AIP with functions $f\not\in AP$ (see Proposition \ref{p4} and Corollary \ref{c5}). On the other hand, $dist_p$ is a probabilistic distance and will enable us to obtain average-case results, {\it i.e.,}  upper bounds for the expected value of the errors introduced by when using the AIP with functions $f\not\in AP$ (see Proposition~\ref{p6}).

\begin{proposition}\label{p4}
Let  $\mathbf{a},\mathbf{b},\mathbf{c}, \mathbf{d}\in\mathbb{R}_+^n$, and $f:D \to \mathbb{R}_+$ such that  $d_{q,\infty}(f,AP_{(\mathbf{p};q)}) \leq \delta$, if $\mathbf{a}:\mathbf{b}::^{\mathbf{p}}\mathbf{c}:\mathbf{d},$ then  $d_q(f(\mathbf{x}),sol_{\mathbf{p}}(f(\mathbf{a}),f(\mathbf{b}),f(\mathbf{c})))\leq\sqrt[q]{4}\delta$.   
\end{proposition}
%\YLN{not $d_q$ but $dist_q$?}

\begin{proof} The proof follows from  simple manipulations that leverage the triangle inequality for the absolute value. For further details see the Appendix.
\end{proof}

In other words, if a function is $\delta$-close to the class  $AP_{(\mathbf{p};q)}$, then the regression errors are at most $\sqrt[q]{4}\delta$. The same applies for the other semimetric $d_q.$

\begin{corollary}
    \label{c5} Let $f:D\to\mathbb{R}_+$ be such that 
    $$d_{q,\infty}(f,AP_{(\mathbf{p};q)})\leq \delta.$$ 
    Then, for every ${x\in D}$, $d_q(f(x),\overline{\mathbf{x}}_{S,f})\leq \sqrt[q]{4}\delta$.
\end{corollary}
%\YLN{$d_{q,\infty}(f,AP_{(\mathbf{p};q)})$? Not $d_{p,\infty}(f,AP_{(\mathbf{p};q)})$?}

\begin{proof}
    Follows from Proposition \ref{p4} and the monotonicity of the $q$-generalized mean $m_q$.
\end{proof}

%\begin{comment}

\subsection{The Probabilistic View}
To obtain a probabilistic counterpart to this result, we will need to introduce the notion of regularity for sample sets.

\begin{definition}
 A sample set $S\subseteq D$ is {\it regular} with respect to a function $f$ if $E_S(f)=D$, and there is  $m\in\mathbb{N}$,  such that$$|R_S(f,\mathbf{x})|=m, ~\text{for every}~ {\mathbf{x}\in D}.$$ 
 \end{definition}

 Intuitively, a regular sample set is a set whose examples are representative and evenly distributed with respect to the feature space and output function. In such sets, nearby inputs lead to nearby outputs, and neighborhoods used for analogical reasoning are well-behaved: neither too sparse nor irregular. This regularity ensures that analogical inferences are stable and meaningful.

Practically, this assumption plays a role similar to the smoothness and density hypotheses in non-parametric regression: it guarantees that local relations in the data approximate the underlying functional dependencies. Regular sample sets thus provide the theoretical grounding ensuring the soundness and convergence of the analogical rule, and empirically correspond to well-sampled regions where analogical reasoning can be applied reliably.

To such a regular sample set, we can associate a mapping $S':D\to M_{3\times m}(S)$ that maps each $\mathbf{x}$ to a matrix 

\begin{displaymath}
    \begin{bmatrix} a^{(\mathbf{x})}_1 & a^{(\mathbf{x})}_2 & \cdots& a^{(\mathbf{x})}_m \\ b^{(\mathbf{x})}_1 & b^{(\mathbf{x})}_2 & \cdots & b^{(\mathbf{x})}_m \\ c^{(\mathbf{x})}_1 & c^{(\mathbf{x})}_2 & \cdots & c^{(\mathbf{x})}_m\end{bmatrix}
\end{displaymath}
where each column is a different element of $R_S(f,\mathbf x)$.

\begin{remark}\label{rem:def17}
    If $S$ is regular, then for every  $\mathbf{x}\in D$, we have 
    \[
    \overline{\mathbf{x}}_{S,f}=\left(\frac{1}{m}\sum_{j=1}^m \left(f(b^{(\mathbf{x})}_j)^q+f(c^{(\mathbf{x})}_j)^q-f(a^{(\mathbf{x})}_j)^q\right)\right)^\frac{1}{q} 
    \]

    Moreover, we have $d_q(\overline{\mathbf{x}}_{S,f},f(\mathbf{x}))$ is equal to
    {%\small
    $$\left(\frac{1}{m}\left|\sum_{j=1}^m \left(f(x)^q+f(a^{(\mathbf{x})}_j)^q-f(b^{(\mathbf{x})}_j)^q-f(c^{(\mathbf{x})}_j)^q\right)\right|\right)^\frac{1}{q}.$$
    }
\end{remark}

The notion of regular sample will enable us to get probabilistic analogue of Proposition \ref{p4}.

\begin{proposition}
    \label{p6}
    Let  $S\subseteq D$ be a  regular sample  set with associated map $S'$, and let  $f,g:D\to\mathbb{R}_+$  such that $g\in AP_{(p_1,...,p_n;q)}$ and $dist_p(f,g)\leq\delta$. Suppose that  $$\mathbb{E}\left(\mathbf{x}\mapsto |f^q-g^q|(S'(\mathbf{x})_{ji})\right)=\mathbb{E}(|f^q-g^q|),$$ for every ${\mathbf{x}\in D}$, ${i\in\{1,...,m\},j\in\{1,2,3\}}$\footnote{Losely speaking, this means that sampling is independent of the $q$-distance between $f$ and $g$.}. Then $$dist_q(\mathbf{x}\mapsto\overline{\mathbf{x}}_{S,f}, f)\leq \sqrt[q]{4}\delta.$$
\end{proposition}
\begin{proof}
    This proof is a matter of unfolding the definitions, using Remark 15, and then applying the triangle inequality (as in Proposition 16) and the linearity of the expected value to the resulting four summands. 
    The detailed proof is given in the extended version. 
\end{proof}
The assumption that $S$ is regular may seem quite  strong. However, the proof of Proposition \ref{p6} shows that we do not necessarily need $S$ to be regular, as long as we can find a suitable value for $m$ to construct a $S':D\to M_{3\times m}(S)$. If we define $  \overline{\mathbf{x}}_{S',f}$ as {\small
$$m_q(sol_q(f(S'(\mathbf{x})_{1\,i},f(S'(\mathbf{x})_{2\,i}),f(S'(\mathbf{x})_{3\,i}))\,|\,i\in{\{1,...,m\}})$$} (which is in general different from $\overline {\mathbf{x}}_{S,f}$ if $S$ is not regular, but still obtainable from $S$ by applying the AIP), Proposition \ref{p6} is still valid if we substitute $\overline {\mathbf{x}}_{S',f}$ for $\overline {\mathbf{x}}_{S,f}$.

%One way to construct this mapping for a (non regular) sample $S$ is to first pick a value for $m$, then reduce the domain $D$ to $$\tilde D:=\{\mathbf{x}\in D\,|\,|R_S(f,\mathbf{x})|\geq m\}$$ and then, for each $\mathbf{x}\in\tilde D$,  suitably reorder the set $R_S(f,\mathbf{x})$ and choosing 
%\begin{eqnarray*}
 %$R_S(f,\mathbf{x})=\{(a_1^{(\mathbf{x})},b_1^{(\mathbf{x})},c_1^{(\mathbf{x})}),...,(a_m^{(\mathbf{x})},b_m^{(\mathbf{x})},c_m^{(\mathbf{x})}),...,(a_{m+k_\mathbf{x}}^{(\mathbf{x})},b_{m+k_\mathbf{x}}^{(\mathbf{x})},c_{m+k_\mathbf{x}}^{(\mathbf{x})})\}$,choose 
% $S'(\mathbf{x})=\begin{bmatrix} a_1^{(\mathbf{x})} & \cdots & a_m^{(\mathbf{x})} \\b_1^{(\mathbf{x})} & \cdots & b_m^{(\mathbf{x})}\\c_1^{(\mathbf{x})}& \cdots & c_m^{(\mathbf{x})} \end{bmatrix}$. \\

 %  \begin{remark}
    Proposition \ref{p6} also assumes that $$\mathbb{E}\left(\mathbf{x}\mapsto |f^q-g^q|(S'(\mathbf{x})_{ji})\right)=\mathbb{E}(|f^q-g^q|),$$ for every ${\mathbf{x}\in D}$, ${i\in\{1,...,m\},j\in\{1,2,3\}}$.
    In other words, the way that  $S'$ is used to find $\overline{ \mathbf{x}}_{S',f}$ does not depend on the distance between $f$ and $g$.
    Since we are establishing an upper bound, this assumption can be relaxed to 
    $$\mathbb{E}\left(\mathbf{x}\mapsto |f^q-g^q|(S'(\mathbf{x})_{ji})\right)\leq\mathbb{E}(\mathbf{x}\mapsto|f^q-g^q|(\mathbf{x})),$$ which essentially states that the sample points we are choosing are more likely to be points where $f$ is closer to $g$, and the proof would remain the same.
    % \end{remark}
    %
  %  If, for example, we pick $S\subseteq \{f=g\}$, we would be able to get the (better) estimate $dist_q(\mathbf{x}\mapsto \overline{\mathbf{x}}_{S'f},f)\leq \delta$ 

%\end{comment}
\section{Boolean Case Revisited}
\label{sec:classification}

%\paragraph{The real case}

%\paragraph{The Boolean case}

We can now attempt to apply the previous results by revisiting  the Boolean case. It has been shown that (for the two models of analogy $M$ and $K$), the analogy preserving functions are precisely the affine functions $\mathbb{B}^n\to\mathbb{B}$ (see Theorem~\ref{AP_is_L} that was obtained in \cite{CouceiroHPR17}).

However, $\mathbb{B}$ is not a subring of $\mathbb{R}$, so some care is required when transferring these results. To illustrate this, take for example the triple $(0,1,1)$. If we see it in $\mathbb{R}$, $sol_1(0,1,1)=2$. But if we see it in $\mathbb{B}$, with model $K$ (see \ref{subsection:to:general:cases}), $sol_1(0,1,1)=0$. Moreover, if our model is $M$, this triple is not even solvable. With this in mind, for the sake of clarity, let us denote the sum in $\mathbb{B}$ as $\oplus$ and the sum in $\mathbb{R}$ as $+$.

Another problem we have to deal with is the fact that in $\mathbb{B}$, $\overline {\mathbf x}_{S,f}\in\mathbb{B}$ was originally defined as a mode. To make this notion compatible with the framework hitherto established, we will need to modify it slightly. Denote as $\iota:\mathbb{B}\to\mathbb{R}$ the set inclusion. For a regular sample $S\subseteq\mathbb{B}^n$, with corresponding map $S':D\to M_{3\times m}(S)$, we can define 
\[
    \tilde{\mathbf x}_{S,f}:=\frac{1}{m}\sum_{i=1}^m \iota \left(\bigoplus_{j=1}^3 f(S'(\mathbf x)_{ji})\right)\in[0,1].
\]
Note that
$\displaystyle\bigoplus_{j=1}^3 f(S'(\mathbf x)_{ji}) = sol(S'(\mathbf x)_{1i},S'(\mathbf x)_{2i},S'(\mathbf x)_{3i}),$
and that $\overline {\mathbf x}_{S,f}\in\mathbb{B}$ can easily be recovered from $\tilde{\mathbf x}_{S,f}$. Moreover, if $f\in AP$, then $\tilde {\mathbf x}_{S,f}\in\{0,1\}$ and $$\tilde {\mathbf x}_{S,f}=\iota(\overline{\mathbf x}_{S,f})=\iota f(\mathbf x).$$ However, in general $f\not\in AP$, and thus $|\tilde {\mathbf x}_{S,f}-\iota(\overline{\mathbf x}_{S,f})|$ can be thought of as a measure of how confident the AIP classifies $\mathbf x$.
% (finds the value of $f(\mathbf x)$).
By observing that, for $a,b\in\mathbb{B}$, 
$$\iota(a\oplus b)=|\iota(a)-\iota(b)|,$$ and the inequality $$\iota(a_1\oplus\cdots\oplus a_n)\leq\iota(a_1)+...+\iota(a_n),$$ we can easily adapt the proof of Proposition \ref{p6} to get the following (analogous) result.

 \begin{proposition} Let $D\subseteq \mathbb{B}^n$, $f,g: D\to\mathbb{B}$ such that $g\in AP=\mathcal{L}$ and $\mathbb{E}(|\iota f-\iota g|)\leq \delta$. Suppose $S$ is regular, and for all ${x\in D}$, ${i\in\{1,...,m\}}$, $j\in\{1,2,3\}$, 
$$\mathbb{E}\left(x\mapsto |\iota f-\iota g|(S'(x)_{ji})\right)=\mathbb{E}(|\iota f-\iota g|).$$
Then $\mathbb{E}(|\mathbf x\mapsto \tilde{\mathbf x}_{S,f}-\iota f|)\leq4\delta.$
\end{proposition}

\begin{remark}
    If the probability measure in use is the normalized counting measure (uniform distribution), $\mathbb{E}(|\iota f-\iota g|)$ is just the (normalized) Hamming distance between $f$ and $g$, used in \cite{CouceiroHPR18}. 
\end{remark}

The remarks following Proposition \ref{p6} also apply here.

\section{Conclusion}
We have revisited analogical inference from a foundational and unifying perspective. Our analysis reveals that previous generalization bounds for analogical classifiers fail even in the Boolean setting. To address this drawback, we explored a parameterized framework based on generalized means, allowing analogical inference to extend naturally to regression tasks. We characterized analogy-preserving functions in this setting and derived both worst-case and average-case guarantees based on functional distances. This framework bridges the gap between Boolean analogical classification and continuous regression, offering a general theory of analogical reasoning across domains. 

While the focus of this paper is theoretical, aiming to invalidate incorrect claims and establish a sound unifying framework, it also sets the stage for empirical evaluation. Our formulation connects with SOTA analogical models (e.g., \cite{BounhasP24}) yet extends them to real-valued outputs through a principled analogical regression rule with formal guarantees (Prop. 16–20). This provides the first foundation for implementing analogical regression against methods such as kernel or k-NN regression.

Our framework yields implementable analogical regression rules with closed-form prediction. It also identifies the class of analogy-preserving functions (Th. 12), enabling learnable parametric models. These results provide both a unified theoretical basis and concrete algorithmic formulations for analogical reasoning in regression tasks. These highlighted new avenues for future work. As mentioned earlier, one of the limitations is the underlying assumption that the analogy model is decomposable with respect to each dimension. Despite not being of utmost importance for the current framework, it prevents taking into account synergetic relations between dimensions and thus a full account of emerging concepts in the analogical reasoning process. Also, we have considered a unified model of analogies based on generalized means. However, it would be worthwhile to investigate other recent variants, {\it e.g.,} based on generalized norms \cite{PradeR24}. 

From a practical perspective, future work will explore empirical evaluation, integration with neural architectures, and extensions to structured and probabilistic analogical forms. 

\newpage
\section{Acknowledgments}
This work was supported by ANR-22-CE23-0002 (ERIANA) and ANR-22-CE23-0023 (AT2TA). Francisco Cunha received support from  Calouste Gulbenkian Foundation under ``Programa Novos Talentos Científicos'' Program.
%\footnote{\url{https://gulbenkian.pt/bolsas-lista/novos-talentos/}}.
% Malheureusement, AAAI interdit l'usage d'hyperref.
Yves Lepage was supported by a research grant from Waseda University, type Tokutei-kadai, number 2025R-034, entitled ``Numerical analogy: mathematical definitions and applications to machine learning''.

%\newpage

\bibliography{aaai2026}

%\begin{comment}
%\begin{comment}
 
\section{Appendix}

%\vspace{2ex}
In this appendix we provide the proofs for the main results of the paper.

\subsection{Proof of Proposition 9}\label{Proof:Prop8}

To show that the condition is sufficient ({\it i.e.,}
$1\,\Rightarrow$\,2), suppose that $f\in AP_{(\mathbf{p};q)}$. Let $\mathbf{x},\mathbf{y},\mathbf{z}, \mathbf{w}\in\mathbb{R}_+^n$ be such that $\mathbf{x}:\mathbf{z}::^{\mathbf{p}}\mathbf{w}:\mathbf{y}$. We want to show that $$f(\mathbf{x}):f(\mathbf{z})::^qf(\mathbf{w}):f(\mathbf{y}).$$
 Suppose first that $f(\mathbf{z})^q-f(\mathbf{x})^q+f(\mathbf{w})^q>0$. Take $T=\{\mathbf{x},\mathbf{z},\mathbf{w}\}$. Then $(\mathbf{x},\mathbf{z},\mathbf{w})\in R_T(f,\mathbf{y})$. 
 
 Since $(\mathbf{x},\mathbf{z},\mathbf{w})\in R_T(f,\mathbf{y})$, it is not difficult to see that we also have
% \Leftrightarrow x_i^{p_j}+y_j^{p_j}=z_j^{p_j}+w_j^{p_j}\Leftrightarrow x_i^{p_j}+y_j^{p_j}=_j^{p_j}+w_j^{p_j}\Leftrightarrow 
$(\mathbf{x},\mathbf{w},\mathbf{z})\in R_T(f,\mathbf{y})$.
 In fact, these are the only triples in $R_T(f,\mathbf{y})$.
 To verify this, suppose first that $(\mathbf{z},\mathbf{x},\mathbf{w})\in R_T(f,\mathbf{y})$. Then, for every $j=1,\ldots,n$, 
 \begin{eqnarray*}
     z_j^{p_j}+y_j^{p_j}&=&x_j^{p_j}+w_j^{p_j}\\&\Rightarrow& 2z_j^{p_j}+y_j^{p_j}~=~x_j^{p_j}+w_j^{p_j}+z_j^{p_j}\\
     &\Rightarrow& 2z_j^{p_j}+y_j^{p_j}~=~2x_j^{p_j}+y_j^{p_j}\\
     &\Rightarrow& z_j^{p_j}~=~x_j^{p_j}.
 \end{eqnarray*}
 Hence, $\mathbf{x}=\mathbf{z}$ and thus $(\mathbf{z},\mathbf{x},\mathbf{w})=(\mathbf{x},\mathbf{z},\mathbf{w})$, which was already in $R_T(f,\mathbf{y})$. % (in the second implication we used the fact that $x_j^{p_j}+w_j^{p_j}=y_j^{p_j}+z_j^{p_j}$). 
 Using a similar argument, one can easily show that for any of the remaining three permutations $\sigma$ on $\{\mathbf{x},\mathbf{z},\mathbf{w}\}$, we have $(\sigma(\mathbf{x}),\sigma(\mathbf{z}),\sigma(\mathbf{w}))\in
 \{(\mathbf{x},\mathbf{z},\mathbf{w}),(\mathbf{x},\mathbf{w},\mathbf{z})\}$.
 We thus conclude that 
 $$R_T(f,\mathbf{y})=\{(\mathbf{x},\mathbf{z},\mathbf{w}),(\mathbf{x},\mathbf{w},\mathbf{z})\},$$ and hence, $\overline{\mathbf{y}}_{T,f}$ is simply
{\small
\begin{equation*}
m_q\left(\sqrt[q]{f(z)^q+f(w)^q-f(x)^q},\sqrt[q]{f(w)^q+f(z)^q-f(x)^q}\right),
\end{equation*}
}
that is, 
$\sqrt[q]{f(z)^q+f(w)^q-f(x)^q}.$     

Since $f\in AP_{(\mathbf{p};q)}$, we know that $f(\mathbf{y})=\overline{\mathbf{y}}_{T,f}$, and hence: 
\begin{eqnarray*}
f(x)^q+f(y)^q&=&f(x)^q+(\overline{y}_{T,f})^q\\
&=&f(x)^q+\Big(\sqrt[q]{f(z)^q+f(w)^q-f(x)^q}\Big)^q\\
&=&f(z)^q+f(w)^q,
\end{eqnarray*}
as desired.

Now, suppose that $f(\mathbf{z})^q-f(\mathbf{x})^q+f(\mathbf{w})^q\leq 0$. As we will see, this can never occur for $f\in AP_{(\mathbf{p};q)}$ continuously differentiable. From Lemma 11, we can assume w.l.o.g. $p_j=1,~\forall_j$.\footnote{The case $p=0$ can be dealt with as in \cite{LepageC24} by relying on exponential and logarithmic transformations.} %but it is easily adaptable for other positive values of $(p_1,...,p_n;q)$.
 
 For the sake of a contradiction, suppose that there are $\mathbf{a},\mathbf{b},\mathbf{c},\mathbf{d}\in\mathbb{R}_+^n$ such that $\mathbf{b}-\mathbf{a}=\mathbf{d}-\mathbf{c}$ and $f(\mathbf{b})-f(\mathbf{a})+f(\mathbf{c})<0$.
     Let $g: \mathbb{R}_+ \to \mathbb{R}$ be given by 
   $g(\mathbf{x}) =f(\mathbf{b})-f(\mathbf{a})+f(\mathbf{x})$.
        Notice that 
        \begin{equation*}\label{eq:*}
        ((\mathbf{b}-\mathbf{a}+\mathbf{x})\in\mathbb{R}_+^n ~\text{and}~ g(\mathbf{x})>0 )\Rightarrow g(\mathbf{x})=f(\mathbf{b}-\mathbf{a}+\mathbf{x}).
        \end{equation*}
        This is due to what we have just shown.

       On the one hand, we trivially have $g(\mathbf{a})=f(\mathbf{b})>0$. On the other hand, $g(\mathbf{c})=f(\mathbf{b})-f(\mathbf{a})+f(\mathbf{c})<0$ by hypothesis. Define $\gamma: [0,1]\to\mathbb{R}_+^n$ by $\gamma(t) =t\mathbf{c}+(1-t)\mathbf{a}$ ($\gamma$ is well defined, since $\mathbb{R}_+^n$ is convex). Moreover, $g\circ\gamma$ is continuous since $f$ and $\gamma$ are.
        By the intermediate value theorem, there is a value $t_0\in ]0,1[ $ such that $g\circ\gamma$ changes sign (from positive to negative) at $t_0$, {\it i.e.,} $g(\gamma(t_0))=0$.%, and for every $\varepsilon>0$, there is $\delta>0$  such that $$t_0<t<t_0+\delta \Rightarrow 0< g(\gamma(t))<\epsilon$$ and $t_0-\delta<t<t_0\Rightarrow -\varepsilon<g(\gamma(t))<0$.

        Since $\gamma(t_0)$ lies in the line segment between $\mathbf{a}$ and $\mathbf{c}$, we also have that $(\mathbf{b}-\mathbf{a}+\gamma(t_0))$ lies in the line segment between $\mathbf{b}-\mathbf{a}+\mathbf{a}=\mathbf{b}$ and $\mathbf{b}-\mathbf{a}+\mathbf{c}=\mathbf{d}$ (by hypothesis). Since $\mathbf{b}$ and $\mathbf{d}$ are both in $\mathbb{R}_+^n$  and $\mathbb{R}_+^n$ is convex, it follows that $(\mathbf{b}-a+\gamma(t_0))$ is in the interior of $\mathbb{R}_+^n$. 
        
        Since $f$ is continuous at $\mathbf{a}-\mathbf{b}+\gamma(t_0)\in\mathbb{R}_+^n$, we must have $f(\mathbf{a}-\mathbf{b}+\gamma(t_0))=g(\gamma(t_0))=0$. Moreover, since $f\circ\gamma$ is continuously differentiable, $f$ also changes sign at   % By continuity, there is a $\delta_0$ such that $|t_0-t|<\delta_0\Rightarrow 
        $(\mathbf{b}-\mathbf{a}+\gamma(t))\in\mathbb{R}_+^n$.

        %Let $\varepsilon>0$. Let $\delta_1$ be such that $$t_0<t<t_0+\delta_1 \Rightarrow 0< g(\gamma(t))<\varepsilon$$ (which exists as shown above). Define  $\delta=\min\{\delta_0,\delta_1\}$, and let $t\in]t_0,t_0+\delta[$.  Since $\delta\leq \delta_0)$, $((b-a+\gamma(t))\in\mathbb{R}_+^n$. Since $\delta\leq\delta_1$, $g(\gamma(t))>0$, by  \eqref{eq:*} we have $$g(\gamma(t))=f(\mathbf{b}-\mathbf{a}+\gamma(t).$$ 
        %We have thus shown that $$\forall_{\varepsilon>0}\exists_{\delta>0} \,\,0<t<t_0+\delta \Rightarrow  0<f(\mathbf{b}-\mathbf{a}+\gamma(t))<\varepsilon,$$ {\it i.e.,} that $\lim_{t\to t_0^+} f(\mathbf{a}-\mathbf{b}+\gamma(t))=0$. 
        %Since $f$ is continuous at $\mathbf{a}-\mathbf{b}+\gamma(t_0)\in\mathbb{R}_+^n$, we must have $f(\mathbf{a}-\mathbf{b}+\gamma(t_0))=0$, which constitutes the desired contradiction since  $$0<f(\mathbf{b}-\mathbf{a}+\gamma(t))<\varepsilon.$$

        ($2\,\Rightarrow\,1$) Suppose that $f$ maps analogies in powers $\mathbf{p}=(p_1,...,p_n)$ to analogies in power $q$. Let $S\subseteq\mathbb{R}_+^n$, and let $\mathbf{y}\in E_S^{\mathbf{p}}(f)$. Let $(\mathbf{x},\mathbf{z},\mathbf{w})\in R_S(f,\mathbf{y})$. Then $\mathbf{x}:\mathbf{z}::\mathbf{w}:\mathbf{y}$, and  $f(\mathbf{x}):f(\mathbf{z})::f(\mathbf{w}):f(\mathbf{y})$, {\it i.e.,} 
        $$\sqrt[q]{f(\mathbf{z})^q+f(\mathbf{w})^q-f(\mathbf{x})^p}=f(\mathbf{y}).$$ Thus, 
        {\small\begin{eqnarray*}
\overline{\mathbf{y}}_{S,f}&=&m_q\left(\sqrt[q]{f(\mathbf{z})^q+f(\mathbf{w})^q-f(\mathbf{x})^q}\,\bigg|\,(\mathbf{x},\mathbf{z},\mathbf{w})\in R_S(f,\mathbf{y})\right)\\
        &=&m_q\big(f(\mathbf{y})\,|\,(\mathbf{x},\mathbf{z},\mathbf{w})\in R_S(f,\mathbf{y})\big)=f(\mathbf{y}).
        \end{eqnarray*}
        }
        Hence, $f\in AP_{(\mathbf{p};q)}$, and the proof of Proposition ~9 is now complete.

\subsection{Proof of Lemma 11 }

 \begin{proof}
        We only need to prove $(1\Rightarrow 2)$ since the converse will follow from choosing $\tilde p_j:= p'_jp_j$ and $\tilde p_j':=\frac{1}{p_j'}$ (same for $q$) and applying $(1\Rightarrow 2)$ with $\tilde p_j$ in the role of $p_j$ and $\tilde p_j'$ in the role of $p_j'$ (notice that this means $\tilde r r\circ f\circ s\tilde s=f$).
        
        Suppose then that $f$ maps analogies in powers $\mathbf{p}=(p_1,\ldots,p_n)$ to analogies in power $q$. Let $\mathbf{x},\mathbf{y},\mathbf{z},\mathbf{w}\in\mathbb{R}_+^n$ such that $$\mathbf{x}:\mathbf{z}::^{{\bf p\odot\bf p'}}\mathbf{w}:\mathbf{y},$$ {\it i.e.,} $$x^{p'_jp_j}_j+y^{p'_jp_j}_j=z^{p'_jp_j}_j+w^{p'_jp_j}_j,$$ for all $j$. We want to show that 
        $$r\circ f\circ s(\mathbf{x})^{q'q}+r\circ f\circ s(\mathbf{y})^{q'q}=r\circ f\circ s(\mathbf{z})^{q'q}+rfs(\mathbf{w})^{q'q}.$$

        Since $f$ maps analogies in powers $(p_1,...,p_n)$ to analogies in power $q$, and since $$(x_j^{p'_j})^{p_j}+(y_j^{p'_j})^{p_j}=(z_j^{p'_j})^{p_j}+(w_j^{p'_j})^{p_j},$$ we know that $f(x_1^{p'_1},...,x^{p'_n})^q+f(y_1^{p'_1},...,y^{p'_n})^q=f(z_1^{p'_1},...,z^{p'_n})^q+f(w_1^{p'_1},...,w^{p'_n})^q,$ {\it i.e.,} 
        $$r\circ f\circ s(x)^{q'q}+r\circ f\circ s(y)^{q'q}=r\circ f\circ s(z)^{q'q}+r\circ f\circ s(w)^{q'q},$$ as we wanted to show.
    \end{proof}

\subsection{Proof of Theorem~12}\label{Subsection:Exeplicit}    

 \begin{proof}
        In light of Proposition 9, we can replace condition $1$ by 
        \begin{itemize}
            \item[$1'$.] $f$ maps analogies in powers $\mathbf{p}$ to analogies in power $q$.
        \end{itemize}
        
        ($1'\Rightarrow 2$) This is the non-trivial side. We will rely on the previously stated lemma, as well as on the well-known fact that a $\mathbb{Z}$-linear continuous function $g:\mathbb{R}^n\to\mathbb{R}$ must in fact be $\mathbb{R}$-linear\footnote{This result is attributed to Cauchy (1821) and known as ``Cauchy's functional equation''.}.
        
        Suppose $f$ maps analogies in powers $\mathbf{p}=(p_1,...,p_n)$ to analogies in power $q$. Then if we define $p'_j:=\frac{1}{p_j}$ and $q':=\frac{1}{q}$, $s$ and $r$ as in Lemma 11, $g=r\circ f\circ s$ maps arithmetic analogies ({\it i.e.,} analogies in powers $\mathbf{1}=(1,...,1)$) to arithmetic analogies.
       
        Define $\pi: \mathbb{R}_+^m\to\mathbb{R}$, by 
        $\pi(\mathbf{x})= g(\mathbf x)-g(\mathbf 0)$.
       It is not difficult to check that $\pi$ is additive, {\it i.e.,} that \[\pi(\mathbf x+\mathbf y)=\pi(\mathbf x)+\pi(\mathbf y),\quad \text{for every }{\mathbf x,\mathbf y\in\mathbb{R}_+^n}.\]
        %Let $\mathbf x,\mathbf y\in\mathbb{R}_{+}^n$. \[g(\mathbf %x+\mathbf y)=rfs(\mathbf x+\mathbf y)-r fs(\mathbf0)\]\[=r fs(\mathbf x)+r\ fs(\mathbf y)-2r fs(\mathbf 0)= g(\mathbf x)+ g(\mathbf y)\]
        %, where in  the middle equality we used the fact that $\mathbf x+\mathbf y=(\mathbf x+\mathbf y)+\mathbf 0\Rightarrow r fs(\mathbf x)+r fs(\mathbf y)=r f s(\mathbf x+\mathbf y)+r fs(\mathbf 0)$.\
        %This proves that $\tilde g$ is additive. 
        By Cauchy's functional equation and the continuity of $\pi$,  $$\pi(\mathbf{x})=\sum_{j=1}^na_jx_j, ~\text{for some $a_1,...,a_n\in\mathbb{R}$.}$$
        Since $\pi+\pi(\mathbf 0)\colon \mathbb{R}_+^n\to \mathbb{R}_+$, we must have $a_j>0\,\,\forall_j$.
        
        Define $b:=\pi(\mathbf 0)$, so that $g=r\circ f\circ s =\pi+b$. Hence, 
        \begin{eqnarray*}
        &&g(\mathbf y)=\pi(\mathbf y)+b ~(\forall {\mathbf y\in\mathbb{R}_+^n})\\
        &\Leftrightarrow& g(x_1^{p_1},...,x_n^{p_n})=\pi(x_1^{p_1},...,x_n^{p_n})+b  ~(\forall {\mathbf x\in\mathbb{R}_+^n})\\
        &\Leftrightarrow&  f(\mathbf x)^q=\sum_{j=1}^na_jx_j^{p_j}+b\\
        &\Leftrightarrow&   f(\mathbf{x})=\sqrt[q]{\sum_{j=1}^na_jx_j^{p_j}+b}
        \end{eqnarray*}
        
        $(2\Rightarrow 1')$
        This is just a matter of checking that these types of functions map analogies to analogies. 
        Let $\mathbf x,\mathbf y,\mathbf  z,\mathbf w\in\mathbb{R}_+^n$ such that for every ${j\in\{1,...,n\}}$, $x_j^{p_j}+y_j^{p_j}=z_j^{p_j}+w_j^{p_j}$.
        
      {\small    \begin{eqnarray*}
%f(\mathbf{x})^{q}+f(\mathbf{y})^{q}&=&\!\!\!\left(\!\left({\sum_{j=1}^n a_{ij}x_j^{p_j}\!+\!b_i}\right)^\frac{1}{q}\!\!\right)^{q}\!\!\!+\left(\!\left({\sum_{j=1}^n a_{ij}y_j^{p_j}\!+\!b_i}\!\!\right)^\frac{1}{q}\!\!\right)^{q}\\
f(\mathbf{x})^{q}+f(\mathbf{y})^{q}&=&\sqrt[q]{\sum_{j=1}^n a_{ij}x_j^{p_j}\!+\!b_i}^{q}\!\!\!+\sqrt[q]{\sum_{j=1}^n a_{ij}y_j^{p_j}\!+\!b_i}^{q}\\
        &=&\sum_{j=1}^n [a_{ij}x_j^{p_j}+b_i]+\sum_{j=1}^n [a_{ij}y_j^{p_j}+b_i]\\
        &=&\sum_{j=1}^n [a_{ij}(x_j^{p_j}+y_j^{p_j})+2b_i]\\
        &=&\sum_{j=1}^n [a_{ij}(z_j^{p_j}+w_j^{p_j})+2b_i]\\
        &=&\sum_{j=1}^n [a_{ij}z_j^{p_j}+b_i]+\sum_{j=1}^n [a_{ij}w_j^{p_j}+b_i]\\
    &=&f(\mathbf{z})^{q}+f(\mathbf{w})^{q}
        \end{eqnarray*}\qedhere
  }  \end{proof}

\subsection{Proof of Proposition 16}
The proof follows the following sequence of (in)equalities:
 {\small 
 \begin{eqnarray*}
 &&d_q(f(x),sol_q(f(a),f(b),f(c)))^q\\
 &=&{|f(x)^q+f(a)^q-f(b)^q-f(c)^q}|\\
 &=&|f(b)^q-g(b)^q+g(b)^q+f(c)^q-g(c)^q+g(c)^q\\
 &-&f(a)^q-g(a)^q+g(a)^q-f(x)^q|\\
 &=&|(f^q-g^q)(b)+(f^q-g^q)(c)-(f^q-g^q)(a)\\
 &+& (g^q(b)+g^q(c)-g^q(a)-f^q(x))|\\
     &\leq&|(f^q-g^q)(a)|+|(f^q-g^q)(b)|+|(f^q-g^q)(c)|+|(f^q-g^q)(x)|\\
     &\leq&||(f^q-g^q)||_\infty+||(f^q-g^q)||_\infty+||(f^q-g^q)||_\infty+||(f^q-g^q)||_\infty\\
 &\leq& 4(d_{q,\infty}(f,g))^q\leq4\delta^q,
 \end{eqnarray*}
 }
         where we used the fact that $g(x)^q=g(b)^q+g(c)^q-g(a)^q$, since $g$ maps analogies in $\mathbf{p}=(p_1,...,p_n)$ to analogies in $q$.

\subsection{Proof of Proposition 20}\label{app:p6}
     
\begin{proof}
    Define $h: D\to\mathbb{R}_+$ by $h(x)= \overline{x}_{S,f}$. Then
    $dist_q(h,f)^q$ is equal to
    \[\mathbb{E}_q(x\mapsto d_q(f(x),h(x))^q=\mathbb{E}_1(x\mapsto d_q(f(x),\overline{x}_{S,f}))^q)\]

    By Remark 15,%\ref{rem:def17},
    we can rewrite it as
  {\small $$\mathbb{E}\left(x\mapsto \frac{1}{m}\left|\sum_{i=1}^m(f(x)^q+f(a^{(x)}_i)^q-f(b^{(x)}_i)^q-f(c^{(x)}_i)^q\right|\right)
  $$
}
    Applying the triangle inequality (for the absolute value), and the monotonicity and linearity of $\mathbb{E}$, we get that $dist_q(h,f)^q$ is upper bounded by {\small}
{\small\begin{equation}
    %\mathbb{E}\left(x\mapsto \frac{1}{m}\sum_{i=1}^m|f(x)^q+f(a^{(x)}_i)^q-f(b^{(x)}_i)^q-f(c^{(x)}_i)^q|\right)=\]
    %\begin{equation}
    \frac{1}{m}\sum_{i=1}^m\mathbb{E}\left(x\mapsto\left|f(x)^q+f(a^{(x)}_i)^q-f(b^{(x)}_i)^q-f(c^{(x)}_i)^q\right|\right)
    \label{eq1}
    \end{equation}
    }
We are thus left with the task of upper-bounding  $$\mathbb{E}\left(x\mapsto\left|f(x)^q+f(a^{(x)}_i)^q-f(b^{(x)}_i)^q-f(c^{(x)}_i)^q\right|\right).$$
Let $\mathbf{x}\in D$, $i\in \{1,...,m\}$. This can be done using the same technique as in the previous proposition. Write $a,b,c$ for $a_i^{(\mathbf{x})},b_i^{(\mathbf{x})},c_i^{(\mathbf{x})}$.
 {\small   
   \begin{eqnarray*}
       \mathbb{E}(x&\mapsto&|f(x)^q+f(a)^q-f(b)^q-f(c)^q|)\\
       =\mathbb{E}(x&\mapsto& |(f^q-g^q)(x)+(f^q-g^q)(a)\\&-&(f^q-g^q)(b)-(f^q-g^q)(c)|)\\
       \leq\mathbb{E}(x&\mapsto& |(f^q-g^q)(x)|+|(f^q-g^q)(a)|\\
       &+&|(f^q-g^q)(b)|+|(f^q-g^q)(c)|)\\
       =\mathbb{E}(x&\mapsto&|(f^q-g^q)(x)|+\mathbb{E}(x\mapsto |(f^q-g^q)(a)|)\\
       +\mathbb{E}(x&\mapsto&|(f^q-g^q)(b)|)+\mathbb{E}(x\mapsto |(f^q-g^q)(c)|)\\ 
       &=& 4\mathbb{E}(|f^q-g^q|)=4dist_q(f,g)^q=4\delta^q 
   \end{eqnarray*}
}
    Substituting this result in \eqref{eq1} we obtain
$$dist_q(h,f)^q\leq \frac{1}{m}\sum_{i=1}^m 4\delta^q$$ which is equivalent to $$ dist_q(x\mapsto\overline{x}_{S,f},f)\leq \sqrt[q]{4}\delta \qedhere$$  
\end{proof}
   
%\end{comment}

\end{document}